\documentclass[10pt,twocolumn,letterpaper]{article}

\usepackage{cvpr}
\usepackage{times}
\usepackage{epsfig}
\usepackage{graphicx}
\usepackage{amsmath}
\usepackage{amssymb}
\usepackage{amsthm}
\usepackage{algorithm,algorithmic}
\usepackage[caption=false,textfont=sf]{subfig}
\usepackage{array}
\usepackage{footnote}

\newcommand*{\affaddr}[1]{#1}
\newcommand*{\affmark}[1][*]{\textsuperscript{#1}}

\DeclareMathOperator*{\argmin}{arg\,min}

\newtheorem{theorem}{Theorem}
\makesavenoteenv{table}

\newcolumntype{L}[1]{>{\raggedright\let\newline\\\arraybackslash\hspace{0pt}}m{#1}}
\newcolumntype{C}[1]{>{\centering\let\newline\\\arraybackslash\hspace{0pt}}m{#1}}
\newcolumntype{R}[1]{>{\raggedleft\let\newline\\\arraybackslash\hspace{0pt}}m{#1}}

\newcommand{\smallsim}{\smallsym{\mathrel}{\sim}}
\makeatletter
\newcommand{\smallsym}[2]{#1{\mathpalette\make@small@sym{#2}}}
\newcommand{\make@small@sym}[2]{%
  \vcenter{\hbox{$\m@th\downgrade@style#1#2$}}%
}
\newcommand{\downgrade@style}[1]{%
  \ifx#1\displaystyle\scriptstyle\else
    \ifx#1\textstyle\scriptstyle\else
      \scriptscriptstyle
  \fi\fi
}

\usepackage[breaklinks=true,bookmarks=false]{hyperref}

\cvprfinalcopy 

\def\cvprPaperID{2575} 

\ifcvprfinal\fi
\begin{document}

\title{Network Sketching: Exploiting Binary Structure in Deep CNNs}

\author{Yiwen Guo\affmark[1], Anbang Yao\affmark[1], Hao Zhao\affmark[2,1], Yurong Chen\affmark[1]\\
\affaddr{\affmark[1]Intel Labs China}, \affaddr{\affmark[2]Department of Electronic Engineering, Tsinghua University}\\
\tt\small {\{yiwen.guo,anbang.yao,hao.zhao,yurong.chen\}@intel.com}\\
}



\maketitle
\thispagestyle{empty}

\begin{abstract}
Convolutional neural networks (CNNs) with deep architectures have substantially advanced the state-of-the-art in computer vision tasks.
However, deep networks are typically resource-intensive and thus difficult to be deployed on mobile devices.
Recently, CNNs with binary weights have shown compelling efficiency to the community, whereas the accuracy of such models is usually unsatisfactory in practice.
In this paper, we introduce network sketching as a novel technique of pursuing binary-weight CNNs, targeting at more faithful inference and better trade-off for practical applications.
Our basic idea is to exploit binary structure directly in pre-trained filter banks and produce binary-weight models via tensor expansion.
The whole process can be treated as a coarse-to-fine model approximation, akin to the pencil drawing steps of outlining and shading.
To further speedup the generated models, namely the sketches, we also propose an associative implementation of binary tensor convolutions.
Experimental results demonstrate that a proper sketch of AlexNet (or ResNet) outperforms the existing binary-weight models by large margins on the ImageNet large scale classification task, while the committed memory for network parameters only exceeds a little.
\end{abstract}

\section{Introduction}\label{sec:int}

Over the past decade, convolutional neural networks (CNNs) have been accepted as the core of many computer vision solutions.
Deep models trained on a massive amount of data have delivered impressive accuracy on a variety of tasks, including but not limited to semantic segmentation, face recognition, object detection and recognition.

In spite of the successes, mobile devices cannot take much advantage of these models, mainly due to their inadequacy of computational resources.
As is known to all, camera based games are dazzling to be operated with object recognition and detection techniques, hence it is eagerly anticipated to deploy advanced CNNs (e.g., AlexNet~\cite{Krizhevsky2012}, VGG-Net~\cite{Simonyan2014} and ResNet~\cite{He2016}) on tablets and smartphones.
Nevertheless, as the winner of ILSVRC-2012 competition, AlexNet comes along with nearly 61 million real-valued parameters and 1.5 billion floating-point operations (FLOPs) to classify an image, making it resource-intensive in different aspects.
Running it for real-time applications would require considerable high CPU/GPU workloads and memory footprints, which is prohibitive on typical mobile devices.
The similar situation occurs on the deeper networks like VGG-Net and ResNet.

Recently, CNNs with binary weights are designed to resolve this problem.
By forcing the connection weights to only two possible values (normally $+$1 and $-$1), researchers attempt to eliminate the required floating-point multiplications (FMULs) for network inference, as they are considered to be the most expensive operations.
In addition, since the real-valued weights are converted to be binary, these networks would commit much less space for storing their parameters, which leads to a great saving in the memory footprints and thus energy costs~\cite{Han2015}.
Several methods have been proposed to train such networks~\cite{Courbariaux2015, Courbariaux2016, Rastegari2016}.
However, the reported accuracy of obtained models is unsatisfactory on large dataset (e.g., ImageNet)~\cite{Rastegari2016}.
Even worse, since straightforwardly widen the networks does not guarantee any increase in accuracy~\cite{Juefei-Xu2016}, it is also unclear how we can make a trade-off between the model precision and expected accuracy with these methods.

\begin{figure}[t]
\begin{center}
\includegraphics[width=0.94\linewidth]{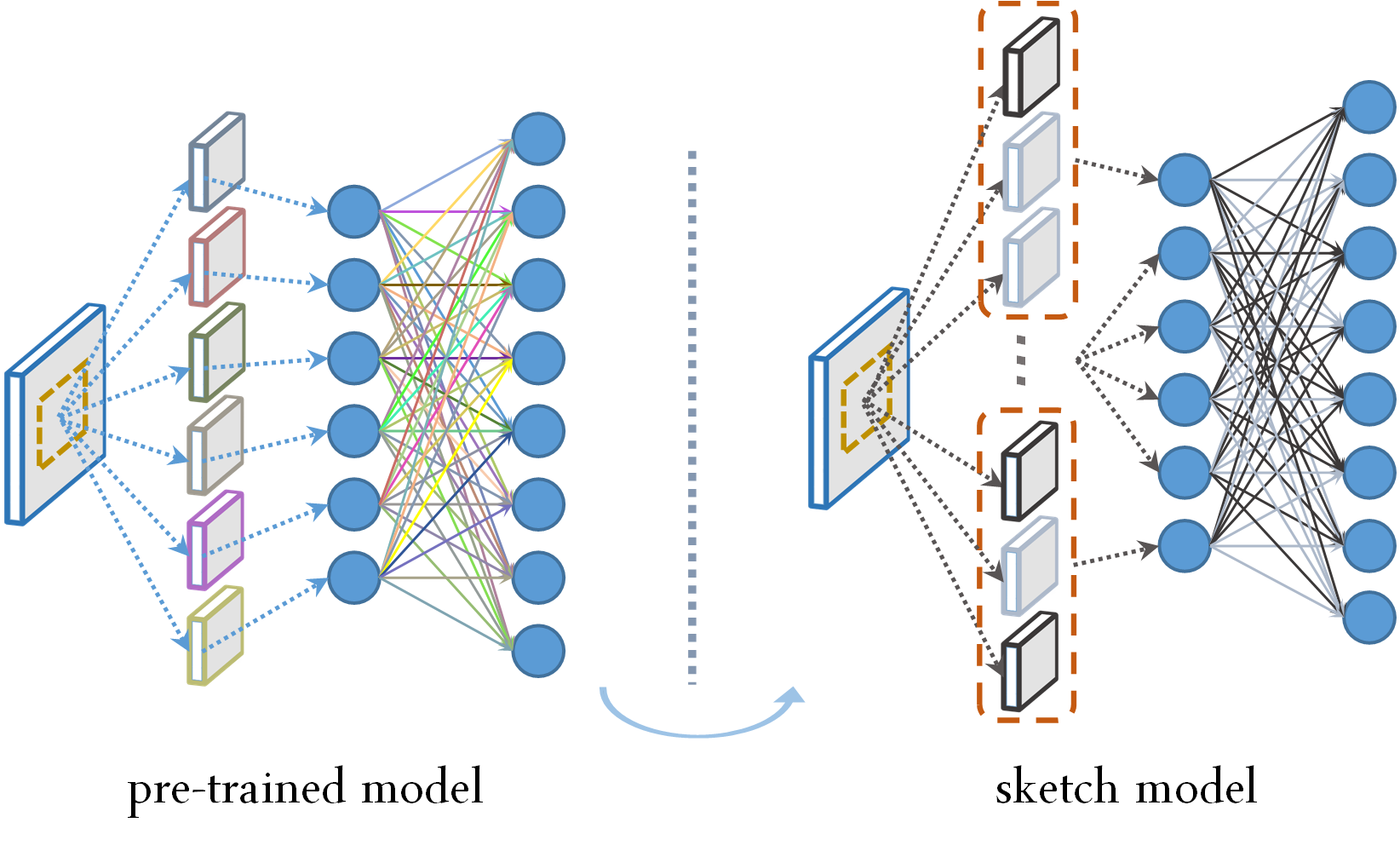}
\end{center}
\vskip -0.1in
\caption{Sketching a network model by exploiting binary structure within pre-trained filter banks, after which the full-precision model can be converted to an efficient one with binary (in black and light grey) connections.}
\label{fig:1}
\vskip -0.1in
\end{figure}

In this paper, we introduce network sketching as a new way of pursuing binary-weight CNNs, where the binary structures are exploited in pre-trained models rather than being trained from scratch.
To seek the possibility of yielding state-of-the-art models, we propose two theoretical grounded algorithms, making it possible to regulate the precision of sketching for more accurate inference.
Moreover, to further improve the efficiency of generated models (a.k.a., sketches), we also propose an algorithm to associatively implement binary tensor convolutions, with which the required number of floating-point additions and subtractions (FADDs)\footnote {Without ambiguity, we collectively abbreviate floating-point additions and floating-point subtractions as FADDs.} is likewise reduced.
Experimental results demonstrate that our method works extraordinarily well on both AlexNet and ResNet.
That is, with a bit more FLOPs required and a little more memory space committed, the generated sketches outperform the existing binary-weight AlexNets and ResNets by large margins, producing near state-of-the-art recognition accuracy on the ImageNet dataset.

The remainder of this paper is structured as follows. In Section~\ref{sec:rel}, we briefly introduce the related works on CNN acceleration and compression.
In Section~\ref{sec:ske}, we highlight the motivation of our method and provide some theoretical analyses for its implementations.
In Section~\ref{sec:spe}, we introduce the associative implementation for binary tensor convolutions.
At last, Section~\ref{sec:exp} experimentally demonstrates the efficacy of our method and Section~\ref{sec:con} draws the conclusions.

\section{Related Works}\label{sec:rel}

The deployment problem of deep CNNs has become a concern for years.
Efficient models can be learnt either from scratch or from pre-trained models.
Generally, training from scratch demands strong integration of network architecture and training policy~\cite{Lin2016}, and here we mainly discuss representative works on the latter strategy.

Early works are usually hardware-specific.
Not restricted to CNNs, Vanhoucke et al.~\cite{Vanhoucke2011} take advantage of programmatic optimizations to produce a $3\times$ speedup on x86 CPUs.
On the other hand, Mathieu et al.~\cite{Mathieu2013} perform fast Fourier transform (FFT) on GPUs and propose to compute convolutions efficiently in the frequency domain.
Additionally, Vasilache et al.~\cite{Vasilache2015} introduce two new FFT-based implementations for more significant speedups.

More recently, low-rank based matrix (or tensor) decomposition has been used as an alternative way to accomplish this task.
Mainly inspired by the seminal works from Denil et al.~\cite{Denil2013} and Rigamonti et al.~\cite{Rigamonti2013}, low-rank based methods attempt to exploit parameter redundancy among different feature channels and filters.
By properly decomposing pre-trained filters, these methods~\cite{Denton2014, Jaderberg2014, Lebedev2014, Zhang2015, Liu2015} can achieve appealing speedups  ($2\times$ to $4\times$) with acceptable accuracy drop ($\leq 1\%$).~\footnote{Some other works concentrate on learning low-rank filters from scratch~\cite{Tai2016, Ioannou2016}, which is out of the scope of our paper.}

Unlike the above mentioned ones, some research works regard memory saving as the top priority.
To tackle the storage issue of deep networks, Gong et al.~\cite{Gong2014}, Wu et al.~\cite{Wu2016} and Lin et al.~\cite{Lin2016} consider applying the quantization techniques to pre-trained CNNs, and trying to make network compressions with minor concessions on the inference accuracy.
Another powerful category of methods in this scope is network pruning.
Starting from the early work of LeCun et al's~\cite{Lecun1989} and Hassibi \& Stork's~\cite{Hassibi1993}, pruning methods have delivered surprisingly good compressions on a range of CNNs, including some advanced ones like AlexNet and VGGnet~\cite{Han2015, Srinivas2015, Guo2016}.
In addition, due to the reduction in model complexity, a fair speedup can be observed as a byproduct.

As a method for generating binary-weight CNNs, our network sketching is orthogonal to most of the existing compression and acceleration methods.
For example, it can be jointly applied with low-rank based methods, by first decomposing the weight tensors into low-rank components and then sketching them.
As for the cooperation with quantization-based methods, sketching first and conducting product quantization thereafter would be a good choice.

\section{Network Sketching}\label{sec:ske}

In general, convolutional layers and fully-connected layers are the most resource-hungry components in deep CNNs. Fortunately, both of them are considered to be over-parameterized~\cite{Denil2013,Veit2016}.
In this section, we highlight the motivation of our method and present its implementation details on the convolutional layers as an example.
Fully-connected layers can be operated in a similar way.

Suppose that the learnable weights of a convolutional layer $\mathcal L$ can be arranged and represented as $\{ \mathbf W^{(i)}:0\leq i < n \}$, in which $n$ indicates the target number of feature maps, and $\mathbf W^{(i)} \in \mathbb R^{c\times w\times h}$ is the weight tensor of its $i$th filter.
Storing all these weights would require $32\times c\times w\times h\times n$ bit memory, and the direct implementation of convolutions would require $s\times c\times w\times h\times n$ FMULs (along with the same number of FADDs), in which $s$ indicates the spatial size of target feature maps.

Since many convolutional models are believed to be informational redundant, it is possible to seek their low-precision and compact counterparts for better efficiency.
With this in mind, we consider exploiting binary structures within $\mathcal L$, by using the divide and conquer strategy.
We shall first approximate the pre-trained filters with a linear span of certain binary basis, and then group the identical basis tensors to pursue the maximal network efficiency.
Details are described in the following two subsections, in which we drop superscript $(i)$ from $\mathbf W$ because the arguments apply to all the $n$ weight tensors.

\subsection{Approximating the Filters}\label{sec:app}

As claimed above, our first goal is to find a binary expansion of $\mathbf W$ that approximates it well (as illustrated in Figure~\ref{fig:2}), which means $\mathbf W \approx \left \langle \mathbf B, \mathbf a \right \rangle = \sum_{j=0}^{m-1} \mathbf a_j \mathbf B_j
$, in which $\mathbf B \in \{+1,-1\}^{c\times w\times h\times m} $ and $\mathbf a \in \mathbb R^m$
are the concatenations of $m$ binary tensors $\{ \mathbf B_0,...,\mathbf B_{m-1} \}$ and the same number of scale factors $\{ \mathbf a_0,...,\mathbf a_{m-1} \}$, respectively.
We herein investigate the appropriate choice of $\mathbf B$ and $\mathbf a$ with a fixed $m$.
Two theoretical grounded algorithms are proposed in Section~\ref{sec:dir} and~\ref{sec:ref}, respectively.

\begin{figure}[ht]
\begin{center}
\includegraphics[width=0.92\linewidth]{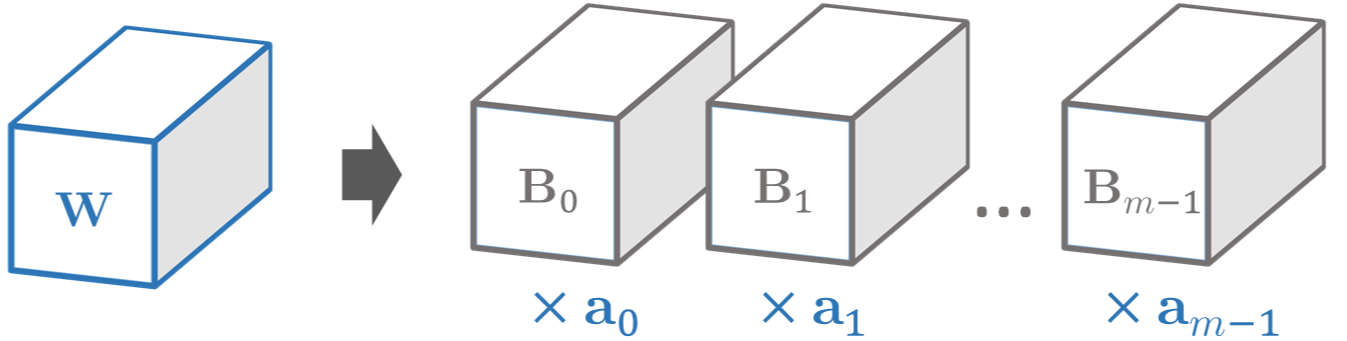}
\end{center}
\caption{Approximate the real-valued weight tensor with a sum of scaled binary tensors.}
\label{fig:2}
\vskip -0.1 in
\end{figure}

\subsubsection{Direct Approximation}\label{sec:dir}

For easy understanding, we shall first introduce the direct approximation algorithm.
Generally, the reconstruction error (or approximation error, round-off error) $e^2:=\| \mathbf W- \langle \mathbf B, \mathbf a \rangle \|^2$ should be minimized to retain the model accuracy after expansion.
However, as a concrete decision problem, directly minimizing $e^2$ seems NP-hard and thus solving it can be very time consuming~\cite{Davis1997}.
In order to finish up in reasonable time, we propose a heuristic algorithm, in which $\mathbf B_j$ and $\mathbf a_j$ are sequentially learnt and each of them is selected to be the current optimum with respect to the $e^2$ minimization criterion.
That is,
\begin{equation}\label{eq:1}
\mathbf B_j, \mathbf a_j = \argmin_{B\in \mathcal B,\ a \in \mathbb R} \ \left \| \hat{\mathbf W}_j - a B \right \|^2,
\end{equation}
in which $\mathcal B := \{+1,-1\}^{c\times w\times h}$, the norm operator $\| \cdot \|$ is defined as $\|\mathbf X\|:=\langle \mathbf X, \mathbf X\rangle^{1/2}$ for any 3-D tensor $\mathbf X$, and $\hat{\mathbf W}_j$ indicates the approximation residue after combining all the previously generated tensor(s).
In particular, $\hat{\mathbf W}_j = \mathbf W$ if $j=0$, and
\begin{equation}\label{eq:2}
\hat{\mathbf W}_j = \mathbf W-\sum_{k=0}^{j-1} \mathbf a_k \mathbf B_k
\end{equation}
if $ j \geq 1$.
It can be easily known that, through derivative calculations, Equation~(\ref{eq:1}) is equivalent with
\begin{equation}\label{eq:3}
\mathbf B_j = \mathrm{sgn}(\hat{\mathbf W}_j)\quad \mathrm{and}\quad  \mathbf a_j=\frac{\langle \mathbf B_j , \hat{\mathbf W}_j \rangle}{t}
\end{equation}
under this circumstance, in which function $\mathrm{sgn}(\cdot)$ calculates the element-wise sign of the input tensor, and $t = c\times w\times h$.

\begin{algorithm}[tbp]
   \caption{The direct approximation algorithm:}
\begin{algorithmic}
   \STATE {\bfseries Input:} $\mathbf {W}$: the pre-trained weight tensor, $m$: the desired cardinality of binary basis. \\
   \STATE{\bfseries Output:} $\{ \mathbf B_j, \mathbf a_j:\ 0\leq j < m \}$: a binary basis and a series of scale factors.\\
   \STATE Initialize $j \leftarrow 0$ and $\hat{\mathbf W}_j\leftarrow \mathbf W$. \\
   \REPEAT
   \STATE Calculate $\mathbf B_j$ and $\mathbf a_j$ by Equation~(\ref{eq:3}).
   \STATE Calculate $\hat{\mathbf W}_{j+1} = \hat{\mathbf W}_j-\mathbf a_j \mathbf B_j$ and update $j \leftarrow j+1$. \\
   \UNTIL{ $j$ reaches its maximal number $m$. }
\end{algorithmic}\label{alg:1}
\end{algorithm}

The above algorithm is summarized in Algorithm~\ref{alg:1}.
It is considered to be heuristic (or greedy) in the sense that each $\mathbf B_j$ is selected to be the current optimum, regardless of whether it will preclude better approximations later on.
Furthermore, some simple deductions give the following theoretical result.

\begin{theorem}\label{theo:1}
For any $m\geq 0$, Algorithm~\ref{alg:1} achieves a reconstruction error $e^2$ satisfying
\begin{equation}\label{eq:4}
e^2 \leq \| \mathbf W \|^2 (1-1/t)^m.
\end{equation}
\end{theorem}
\begin{proof}
Since $\mathbf B_j = \mathrm{sgn}(\hat{\mathbf W}_j)$, we can obtain that,
\begin{equation}\label{eq:5}
\langle \mathbf B_j, \hat{\mathbf W}_j \rangle = \sum \nolimits_l | \hat{w}^{(l)}_j | \geq \| \hat{\mathbf W}_j \|,
\end{equation}
in which $\hat{w}_j$ is an entry of $\hat{\mathbf W}_j$, with superscript $(l)$ indicates its index. From Equation~(\ref{eq:2}) and~(\ref{eq:5}), we have
\begin{equation}\label{eq:6}
\begin{split}
\| \hat{\mathbf W}_{j+1} \|^2 & = \|\hat{\mathbf W}_j \|^2 - \mathbf a_j \langle \mathbf B_j, \hat{\mathbf W}_j \rangle \\
& = \|\hat{\mathbf W}_j \|^2 \left (1 - \frac{\langle \mathbf B_j, \hat{\mathbf W}_j \rangle^2}{ t\|\hat{\mathbf W}_j \|^2} \right ) \\
& \leq \|\hat{\mathbf W}_j \|^2 \left (1 - 1/t \right ).
\end{split}
\end{equation}
The result follows by applying Formula~(\ref{eq:6}) for $j$ varying from $0$ to $m-1$.
\end{proof}

\subsubsection{Approximation with Refinement}\label{sec:ref}

We can see from Theorem~\ref{theo:1} that, by utilizing the direct approximation algorithm, the reconstruction error $e^2$ decays exponentially with a rate proportional to $1/t$.
That is, given a $\mathbf W$ with small size (i.e., when $t$ is small), the approximation in Algorithm~\ref{alg:1} can be pretty good with only a handful of binary tensors.
Nevertheless, when $t$ is relatively large, the reconstruction error will decrease slowly and the approximation can be unsatisfactory even with a large number of binary tensors.
In this section, we propose to refine the direct approximation algorithm for better reconstruction property.

Considering that, in Algorithm~\ref{alg:1}, both $\mathbf B_j$ and $\mathbf a_j$ are chosen to minimize $e^2$ with fixed counterparts.
However, in most cases, it is doubtful whether $\mathbf B$ and $\mathbf a$ are optimal overall.
If not, we may need to refine at least one of them for the sake of better approximation.
On account of the computational simplicity, we turn to a specific case when $\mathbf B$ is fixed.
That is, suppose the direct approximation has already produced $\hat{\mathbf B}$ and $\hat{\mathbf a}$, we hereby seek another scale vector $\mathbf a$ satisfying $\|\mathbf W-\langle \hat{\mathbf B},\mathbf a \rangle \|^2 \leq \|\mathbf W-\langle \hat{\mathbf B},\hat{\mathbf a} \rangle \|^2$.
To achieve this, we follow the least square regression method and get
\begin{equation}\label{eq:7}
\mathbf a_j = \left ( B_j^T B_j \right)^{-1} B_j^T \cdot \mathrm{vec}(\mathbf W),
\end{equation}
in which the operator $\mathrm{vec}(\cdot)$ gets a column vector whose elements are taken from the input tensor, and $B_j$ gets $[ \mathrm{vec}(\mathbf B_{0}),...,\mathrm{vec}(\mathbf B_{j})]$.

\begin{algorithm}[tbp]
   \caption{Approximation with refinement:}
\begin{algorithmic}
   \STATE {\bfseries Input:} $\mathbf {W}$: the pre-trained weight tensor, $m$: the desired cardinality of binary basis. \\
   \STATE{\bfseries Output:} $\{ \mathbf B_j, \mathbf a_j: 0\leq j < m \}$: a binary basis and a series of scale factors.\\
   \STATE Initialize $j \leftarrow 0$ and $\hat{\mathbf W}_j\leftarrow \mathbf W$. \\
   \REPEAT
   \STATE Calculate $\mathbf B_j$ and $\mathbf a_j$ by Equation~(\ref{eq:3}) and~(\ref{eq:7}). \\
   \STATE Update $j \leftarrow j+1$ and calculate $\hat{\mathbf W}_j$ by Equation~(\ref{eq:2}). \\
   \UNTIL{ $j$ reaches its maximal number $m$. }
\end{algorithmic}\label{alg:2}
\end{algorithm}

The above algorithm with scale factor refinement is summarized in Algorithm~\ref{alg:2}.
Intuitively, the refinement operation attributes a memory-like feature to our method, and the following theorem ensures it can converge faster in comparison with Algorithm~\ref{alg:1}.
\begin{theorem}\label{theo:2}
For any $m\geq 0$, Algorithm~\ref{alg:2} achieves a reconstruction error $e^2$ satisfying
\begin{equation}\label{eq:8}
e^2 \leq \left \| \mathbf W \right \|^2 \prod_{j=0}^{m-1} \left (1-\frac{1}{t-\lambda(j,t)} \right ),
\end{equation}
in which $\lambda(j,t)\geq0$, for $0 \leq j \leq m-1$.
\end{theorem}
\begin{proof}
To simplify the notations, let us further denote $w_j := \mathrm{vec}(\mathbf W_j)$ and $b_{j+1} := \mathrm{vec}(\mathbf B_{j+1})$, then we can obtain by block matrix multiplication and inverse that,
\begin{equation}
(B_{j+1}^T B_{j+1})^{-1} =\begin{bmatrix}
\Phi+\Phi \psi \psi^T\Phi/r & -\Phi \psi/r \\
-\psi^T \Phi/r & 1/r
\end{bmatrix},
\end{equation}
in which $\Phi = (B_j^T B_j)^{-1}$, $\psi = B_{j}^T b_{j+1}$, and $r = b_{j+1}^T b_{j+1}-\psi^T \Phi \psi$.
Consequently, we have the following equation for $j=0,...,m-1$,
\begin{equation}\label{eq:ske11}
w_{j+1}=\left ( I- \frac{\Lambda (b_{j+1} b_{j+1}^T)}{b_{j+1}^T \Lambda b_{j+1}} \right ) w_j,
\end{equation}
by defining $\Lambda := I-B_j\Phi B_j^T$. As we know, given positive semi-definite matrices $X$ and $Y$, $\mathrm{tr}(XY)\leq \mathrm{tr}(X)\mathrm{tr}(Y)$. Then, Equation~(\ref{eq:ske11}) gives,
\begin{equation}\nonumber
\begin{split}
\| \hat{\mathbf W}_{j+1} \|^2 & \leq \| \hat{\mathbf W}_j \|^2 - \frac{w_j^T (b_{j+1} b_{j+1}^T) \Lambda (b_{j+1} b_{j+1}^T) w_j}{ (b_{j+1}^T \Lambda b_{j+1})^2} \\
& = \| \hat{\mathbf W}_j \|^2 - \frac{w_j^T (b_{j+1} b_{j+1}^T) w_j}{b_{j+1}^T \Lambda b_{j+1}} \\
& = \|\hat{\mathbf W}_j \|^2 \left (1 - \frac{\langle \mathbf B_j, \hat{\mathbf W}_j \rangle^2}{b_{j+1}^T \Lambda b_{j+1} \|\hat{\mathbf W}_j \|^2 } \right ).
\end{split}
\end{equation}
Obviously, it follows that,
\begin{equation}
\| \hat{\mathbf W}_{j+1} \|^2 \leq \|\hat{\mathbf W}_j \|^2 (1-1/(t-\lambda(j,t))),
\end{equation}
in which $\lambda (j,t) := b_{j+1}^T (I-\Lambda) b_{j+1}$.
Since $\lambda (j,t)$ represents the squared Euclidean norm of an orthogonal projection of $b_{j+1}$, it is not difficult to prove $\lambda(j,t)\geq0$, and then the result follows.
\end{proof}

\subsection{Geometric Interpretation}\label{sec:exp}

After expanding the pre-trained filters, we can group the identical binary tensors to save some more memory footprints.
In this paper, the whole technique is named as network sketching, and the generated binary-weight model is straightforwardly called a sketch.
Next we shall interpret the sketching process from a geometric point of view.

For a start, we should be aware that, Equation~(\ref{eq:1}) is essentially seeking a linear subspace spanned by a set of $t$-dimensional binary vectors to minimize its Euclidean distance to $\mathrm{vec}(\mathbf W)$.
In concept, there are two variables to be determined in this problem.
Both Algorithm~\ref{alg:1} and~\ref{alg:2} solve it in a heuristic way, and the $j$th binary vector is always estimated by minimizing the distance between itself and the current approximation residue.
What make them different is that, Algorithm~\ref{alg:2} takes advantage of the linear span of its previous $j-1$ estimations for better approximation, whereas Algorithm~\ref{alg:1} does not.

Let us now take a closer look at Theorem~\ref{theo:2}.
Compared with Equation~(\ref{eq:4}) in Theorem~\ref{theo:1}, the distinction of Equation~(\ref{eq:8}) mainly lies in the existence of $\lambda(j,t)$.
Clearly, Algorithm~\ref{alg:2} will converge faster than Algorithm~\ref{alg:1} as long as $\lambda(j,t)>0$ holds for any $j \in [0,m-1]$.
Geometrically speaking, if we consider $B_j(B_j^T B_j)^{-1} B_j^T$ as the matrix of an orthogonal projection onto $\mathcal S_j:=\mathrm{span}(b_0,...,b_j)$, then $\lambda(j,t)$ is equal to the squared Euclidean norm of a vector projection.
Therefore, $\lambda(j,t)=0$ holds if and only if vector $b_{j+1}$ is orthogonal to $\mathcal S_j$, or in other words, orthogonal to each element of $\{b_0,...,b_j\}$ which occurs with extremely low probability and only on the condition of $t\in\{2k: k\in \mathbb N\}$.
That is, Algorithm~\ref{alg:2} will probably prevail over Algorithm~\ref{alg:1} in practice.

\section{Speeding-up the Sketch Model}\label{sec:spe}

Using Algorithm~\ref{alg:1} or~\ref{alg:2}, one can easily get a set of $mn$ binary tensors on $\mathcal L$, which means the storage requirement of learnable weights is reduced by a factor of $32t/(32m+tm)\times$.
When applying the model, the required number of FMULs is also significantly reduced, by a factor of $(t/m)\times$.
Probably, the only side effect of sketching is some increases in the number of FADDs, as it poses an extra burden on the computing units.

In this section, we try to ameliorate this defect and introduce an algorithm to further speedup the binary-weight networks.
We start from an observation that, yet the required number of FADDs grows monotonously with $m$, the inherent number of addends and augends is always fixed with a given input of $\mathcal L$.
That is, some repetitive FADDs exist in the direct implementation of binary tensor convolutions.
Let us denote $\mathbf X\in \mathbb R^{c\times w\times h}$ as an input sub-feature map and see Figure~\ref{fig:3} for a schematic illustration.

\begin{figure}[ht]
\begin{center}
\includegraphics[width=0.92\linewidth]{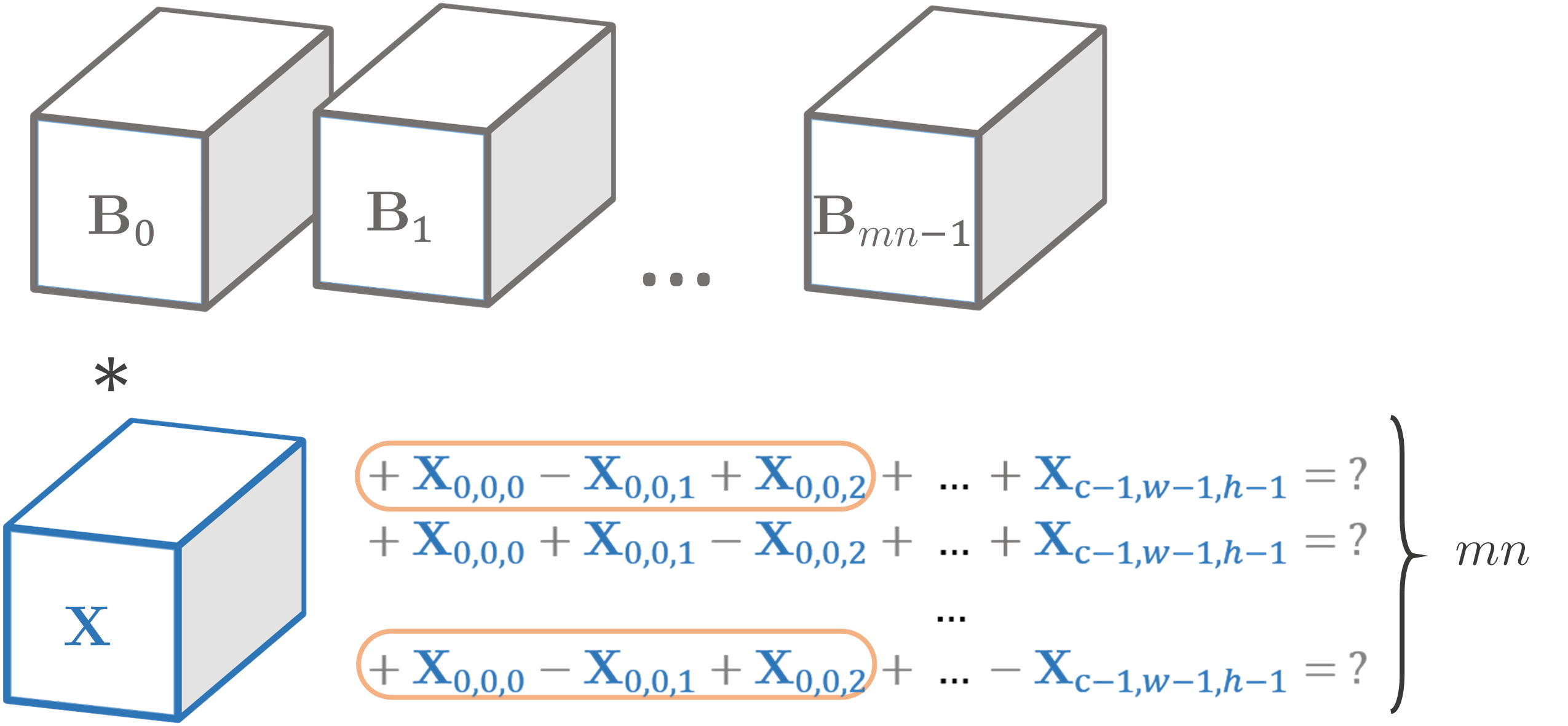}
\end{center}
\caption{As highlighted in the rounded rectangles, with high probability, repetitive FADD exists in the direct implementation of binary tensor convolutions.}
\label{fig:3}
\end{figure}

\subsection{Associative Implementation}\label{sec:ass}

To properly avoid redundant operations, we first present an associative implementation of the multiple convolutions $\mathbf X\ast \mathbf B_0,...,\mathbf X\ast \mathbf B_{mn-1}$ on $\mathcal L$, in which the connection among different convolutions is fully exploited.
To be more specific, our strategy is to perform convolutions in a hierarchical and progressive way.
That is, each of the convolution results is used as a baseline of the following calculations.
Suppose the $j_0$-th convolution is calculated in advance and it produces $\mathbf X \ast \mathbf B_{j_0}=s$, then the convolution of $\mathbf X$ and $\mathbf B_{j_1}$ can be derived by using
\begin{equation}\label{eq:12}
\mathbf X \ast \mathbf B_{j_1} = s+(\mathbf X \ast (\mathbf B_{j_0} \veebar \mathbf B_{j_1}) )\times2,
\end{equation}
or alternatively,
\begin{equation}\label{eq:13}
\mathbf X \ast \mathbf B_{j_1} =s-(\mathbf X \ast (\neg \mathbf B_{j_0} \veebar \mathbf B_{j_1}) )\times 2,
\end{equation}
in which $\neg$ denotes the element-wise not operator, and $\veebar$ denotes an element-wise operator whose behavior is in accordance with Table~\ref{tab:1}.

\begin{table}
\begin{center}
\begin{tabular}{|c|c|c|}
\hline
$\mathbf B_{j_1}$ & $\mathbf B_{j_2}$ & $\mathbf B_{j_1} \vee \mathbf B_{j_2}$  \\
\hline \hline
$+1$ & $-1$  & $-1$\\
$+1$ & $+1$ & $0$ \\
$-1$  & $-1$  & $0$\\
$-1$  & $+1$  & $+1$\\
\hline
\end{tabular}
\end{center}
\caption{Truth table of the element-wise operator $\vee$.}\label{tab:1}
\end{table}

Since $\mathbf B_{j_0} \veebar \mathbf B_{j_1}$ produces ternary outputs on each index position, we can naturally regard $\mathbf X \ast (\mathbf B_{j_0} \veebar \mathbf B_{j_1})$ as an iteration of \texttt{switch} ... \texttt{case} ...  statements.
In this manner, only the entries corresponding to $\pm1$ in $\mathbf B_{j_0} \veebar \mathbf B_{j_1}$ need to be operated in $\mathbf X$, and thus acceleration is gained.
Assuming that the inner-product of $\mathbf B_{j_0}$ and $\mathbf B_{j_1}$ equals to $r$, then $(t-r)/2+1$ and $(t+r)/2+1$ FADDs are still required for calculating Equation~(\ref{eq:12}) and~(\ref{eq:13}), respectively.
Obviously, we expect the real number $r\in [-t,+t]$ to be close to either $t$ or $-t$ for the possibility of fewer FADDs, and thus faster convolutions in our implementation.
In particular, if $r\geq0$, Equation~(\ref{eq:12}) is chosen for better efficiency; otherwise, Equation~(\ref{eq:13}) should be chosen.

\subsection{Constructing a Dependency Tree}\label{sec:dep}

Our implementation works by properly rearranging the binary tensors and implementing binary tensor convolutions in an indirect way.
For this reason, along with Equations~(\ref{eq:12}) and~(\ref{eq:13}), a dependency tree is also required to drive it.
In particular, dependency is the notion that certain binary tensors are linked to specify which convolution to perform first and which follows up.
For instance, with the depth-first-search strategy, $\mathcal T$ in Figure~\ref{fig:4} shows a dependency tree indicating first to calculate $\mathbf X\ast \mathbf B_1$, then to derive $\mathbf X\ast \mathbf B_0$ from the previous result, then to calculate $\mathbf X\ast \mathbf B_3$ on the base of $\mathbf X\ast \mathbf B_0$, and so forth.
By traversing the whole tree, all $mn$ convolutions will be progressively and efficiently calculated.

\begin{figure}[ht]
\begin{center}
\includegraphics[width=0.72\linewidth]{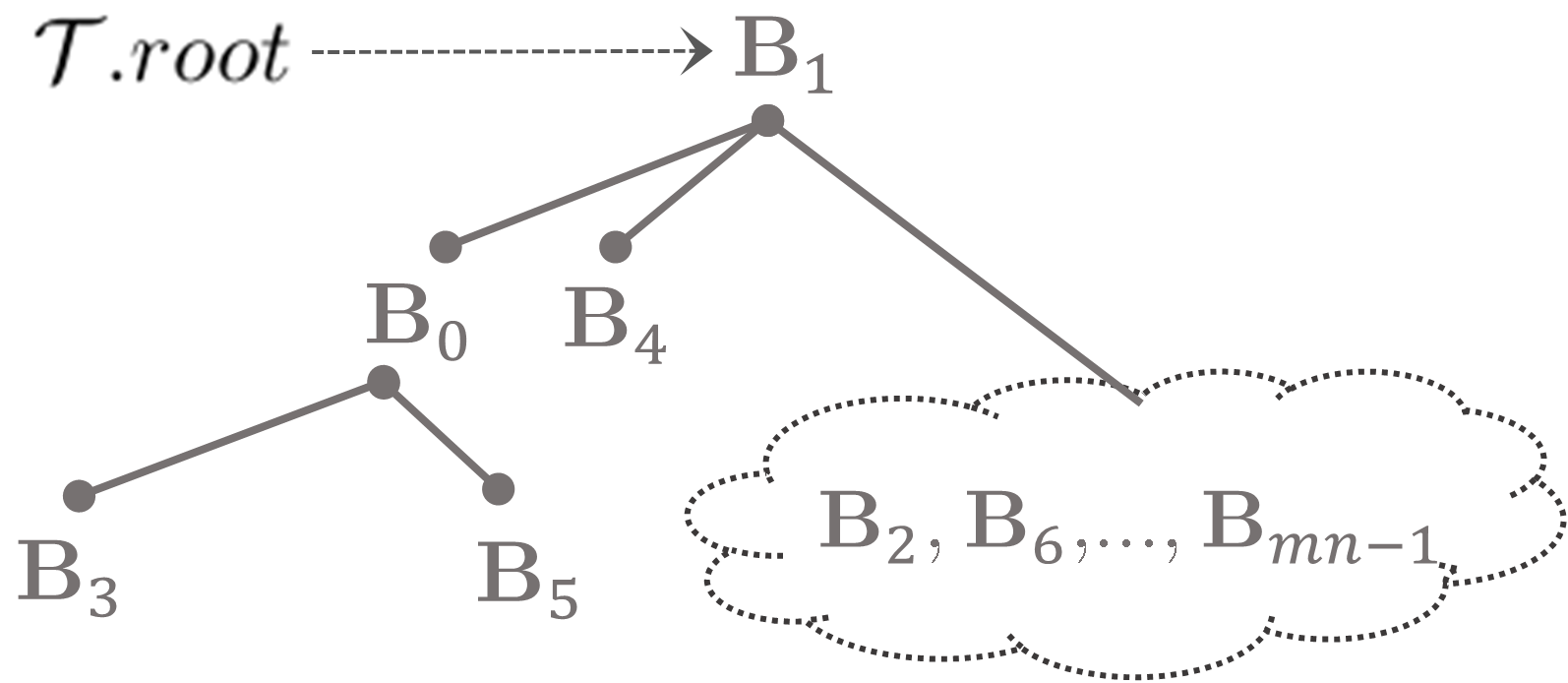}
\end{center}
\caption{A dependency tree for our algorithm. It suggests an order under which the associative convolutions are to be performed.}
\label{fig:4}
\end{figure}

In fact, our algorithm works with a stochastically given tree, but a dedicated $\mathcal T$ is still in demand for its optimum performance.
Let $G=\{V,E\}$ be an undirected weighted graph with the vertex set $V$ and weight matrix $E \in \mathbb R^{mn\times mn}$.
Each element of $V$ represents a single binary tensor, and each element of $E$ measures the distance between two chosen tensors.
To keep in line with the previous subsection, we here define the distance function of the following form
\begin{equation}
d(\mathbf B_{j_0},\mathbf B_{j_1}):=\min \left(\frac{t+r}{2}, \frac{t-r}{2} \right),
\end{equation}
in which $r=\langle \mathbf B_{j_0}, \mathbf B_{j_1}\rangle$ indicates the inner-product of $\mathbf B_{j_0}$ and $\mathbf B_{j_1}$.
Clearly, the defined function is a metric on $\{-1,+1\}^{c\times w\times h}$ and its range is restricted in $[0,t]$.
Recall that, we expect $r$ to be close to $\pm t$ in the previous subsection.
In consequence, the optimal dependency tree should have the shortest distance from root to each of its vertices, and thus the minimum spanning tree (MST) of $G$ is what we want.

From this perspective, we can use some off-the-shell algorithms to construct such a tree.
Prim's algorithm~\cite{Prim1957} is chosen in this paper on account of its linear time complexity with respect to $|E|$, i.e., $O(m^2n^2)$ on $\mathcal L$.
With the obtained $\mathcal T$, one can implement our algorithm easily and the whole process is summarized in Algorithm~\ref{alg:3}.
Note that, although the fully-connected layers calculate vector-matrix multiplications, they can be considered as a bunch of tensor convolutions.
Therefore, in the binary case, we can also make accelerations in the fully-connected layers by using Algorithm~\ref{alg:3}.

\begin{algorithm}[tbp]
   \caption{The associative implementation:}
\begin{algorithmic}
   \STATE {\bfseries Input:} $\{\mathbf {B}_j: 0\leq j<mn\}$: the set of binary tensors, $\mathbf X$: the input sub-feature map, $\mathcal T$: the dependency tree. \\
   \STATE{\bfseries Output:} $\{y_j:0\leq j<mn\}$: the results of convolution.\\
   \STATE Get $z =\mathcal T.root$ and calculate $y_{z.key} = \mathbf X \ast \mathbf B_{z.key}$.\\
   \STATE Initialize the baseline value by $s\leftarrow y_{z.key}$.
   \REPEAT
   \STATE Search the next node of $\mathcal T$ and update $z$, $s$.
   \STATE Calculate $y_{z.key}$ by using Equation~(\ref{eq:12}) or~(\ref{eq:13}). \\
   \UNTIL{ search ends. }
\end{algorithmic}\label{alg:3}
\end{algorithm}

\section{Experimental Results}\label{sec:exp}

In this section, we try to empirically analyze the proposed algorithms.
For pragmatic reasons, all experiments are conducted on the famous ImageNet ILSVRC-2012 database~\cite{ILSVRC15} with advanced CNNs and the open-source Caffe framework~\cite{Jia2014}.
The training set  is comprised of 1.2 million labeled images and the test set is comprised of 50,000 validation images.

In Section~\ref{sec:exp1} and~\ref{sec:exp2}, we will test the performance of the sketching algorithms (i.e., Algorithm~\ref{alg:1} and~\ref{alg:2}) and the associative implementation of convolutions (i.e., Algorithm~\ref{alg:3}) in the sense of filter approximation and computational efficiency, respectively.
Then, in Section~\ref{sec:exp3}, we evaluate the whole-net performance of our sketches and compare them with other binary-weight models.

\subsection{Efficacy of Sketching Algorithms}\label{sec:exp1}

As a starting experiment, we consider sketching the famous AlexNet model~\cite{Krizhevsky2012}.
Although it is the champion solution of ILSVRC-2012, AlexNet seems to be very resource-intensive.
Therefore, it is indeed appealing to seek its low-precision and efficient counterparts.
As claimed in Section~\ref{sec:int}, AlexNet is an 8-layer model with 61M learnable parameters.
Layer-wise details are shown in Table~\ref{tab:2}, and the pre-trained reference model is available online~\footnote{\url{https://github.com/BVLC/caffe/tree/master/models/bvlc_alexnet}.}.

\begin{table}
\begin{center}
\begin{tabular}{|c|c|c|C{0.55in}|}
\hline
Layer Name & Filters & Params (b) & FLOPs   \\
\hline \hline
Conv1 & 96   & $\smallsim$1M   & $\smallsim$211M\\
Conv2 & 256 & $\smallsim$10M & $\smallsim$448M\\
Conv3 & 384 & $\smallsim$28M & $\smallsim$299M\\
Conv4 & 384 & $\smallsim$21M & $\smallsim$224M\\
Conv5 & 256 & $\smallsim$14M & $\smallsim$150M\\
Fc6 & 1 & $\smallsim$1208M & $\smallsim$75M\\
Fc7 & 1 & $\smallsim$537M & $\smallsim$34M\\
Fc8 & 1 & $\smallsim$131M & $\smallsim$8M\\
\hline
\end{tabular}
\end{center}
\caption{Details of the learnable layers in AlexNet~\cite{Krizhevsky2012}, in which "Conv2" is the most computationally expensive one and "Fc6" commits the most memory (in bits). In all these layers, FLOPs consist of the same number of FADDs and FMULs.}\label{tab:2}
\end{table}

Using Algorithm~\ref{alg:1} and~\ref{alg:2}, we are able to generate binary-weight AlexNets with different precisions.
Theoretical analyses have been given in Section~\ref{sec:ske}, so in this subsection, we shall empirically analyze the proposed algorithms.
In particular, we demonstrate in Figure~\ref{fig:5} how "energy" accumulates with a varying size of memory commitment for different approximators.
Defined as $1-\sum e^2/\sum\|W\|^2$, the accumulative energy is negatively correlated with reconstruction error~\cite{Zhang2015}, so the faster it increases, the better.
In the figure, our two algorithms are abbreviately named as "Sketching (direct)" and "Sketching (refined)".
To compare with other strategies, we also test the stochastically generated binary basis (named "Sketching\_random") as used in~\cite{Juefei-Xu2016}, and the network pruning technique~\cite{Han2015} which is not naturally orthogonal with our sketching method.
The scalar factors for "Sketching (random)" is calculated by Equation~(\ref{eq:7}) to ensure its optimal performance.

We can see that, it is consistent with the theoretical result that Algorithm~\ref{alg:1} converges much slower than Algorithm~\ref{alg:2} on all learnable layers, making it less effective on the filter approximation task.
However, on the other hand, Algorithm~\ref{alg:1} can be better when compared with "Sketching (random)" and the pruning technique.
With small working memory, our direct approximation algorithm approximates better.
However, if the memory size increases, pruning technique may converge faster to its optimum.

As discussed in Section~\ref{sec:spe}, parameter $m$ balances the model accuracy and efficiency in our algorithms.
Figure~\ref{fig:5} shows that, a small $m$ (for example 3) should be adequate for AlexNet to attain over 80\% of the accumulative energy in its refined sketch.
Let us take layer "Conv5" and "Fc6" as examples and see Table~\ref{tab:3} for more details.

\begin{table}[htb]
\begin{center}
\begin{tabular}{|C{0.8in}|C{0.65in}|C{0.6in}|C{0.5in}|}
\hline
Layer Name & Energy (\%) & Params (b) & FMULs   \\
\hline \hline
Conv2\_sketch & 82.9 & $\smallsim$0.9M & $\smallsim$560K  \\
Fc6\_sketch & 94.0 & $\smallsim$114M & $\smallsim$12K  \\
\hline
\end{tabular}
\end{center}
\caption{With only 3 bits allocated, the refined sketch of AlexNet attains over 80\% of  the energy on "Conv2" and "Fc6", and more than \textbf{10}$\times$ reduction in the committed memory for network parameters. Meanwhile, the required number of FMULs is also extremely reduced (by \textbf{400}$\times$ and $\sim$\textbf{3000}$\times$) on the two layers.}\label{tab:3}
\end{table}

\subsection{Efficiency of Associative Manipulations}\label{sec:exp2}

The associative implementation of binary tensor manipulations (i.e., Algorithm~\ref{alg:3}) is directly tested on the 3-bit refined sketch of AlexNet.
To begin with, we still focus on "Conv2\_sketch" and "Fc6\_sketch".
Just to be clear, we produce the result of Algorithm~\ref{alg:3} with both a stochastically generated dependency tree and a delicately calculated MST, while the direct implementation results are compared as a benchmark.
All the implementations require the same number of FMULs, as demonstrated in Table~\ref{tab:3}, and significantly different number of FADDs, as compared in Table~\ref{tab:4}.
Note that, in the associative implementations, some logical evaluations and $\times 2$ operations are specially involved.
Nevertheless, they are much less expensive than the FADDs and FMULs~\cite{Rastegari2016}, by at least an order of magnitude.
Hence, their cost are not deeply analyzed in this subsection~\footnote{Since the actual speedups varies dramatically with the architecture of processing units, so we will not measure it in the paper.}.

\begin{table}[ht]
\begin{center}
\begin{tabular}{|C{1.2in}|C{0.8in}|C{0.68in}|}
\hline
Implementation & Conv2\_sketch & Fc6\_sketch   \\
\hline \hline
Direct & $\smallsim$672M & $\smallsim$113M  \\
Associative (random) & $\smallsim$328M & $\smallsim$56M  \\
Associative (MST) & $\smallsim$265M & $\smallsim$49M  \\
\hline
\end{tabular}
\end{center}
\caption{The associative implementation remarkably reduces the required number of FADDs on "Conv2\_sketch" and "Fc6\_sketch".}\label{tab:4}
\end{table}

From the above table, we know that our associative implementation largely reduces the required number of FADDs on "Conv2\_sketch" and "Fc6\_sketch".
That is, it properly ameliorates the adverse effect of network sketching and enables us to evaluate the 3-bit sketch of AlexNet without any unbearably increase in the required amount of computation.
In addition, the MST helps to further improve its performance and finally get $\sim$\textbf{2.5}$\times$ and $\sim$ \textbf{2.3}$\times$ reductions on the two layers respectively.
Results on all learnable layers are summarized in Figure~\ref{fig:6}.

\begin{figure}[t]
\includegraphics[width=0.87\linewidth]{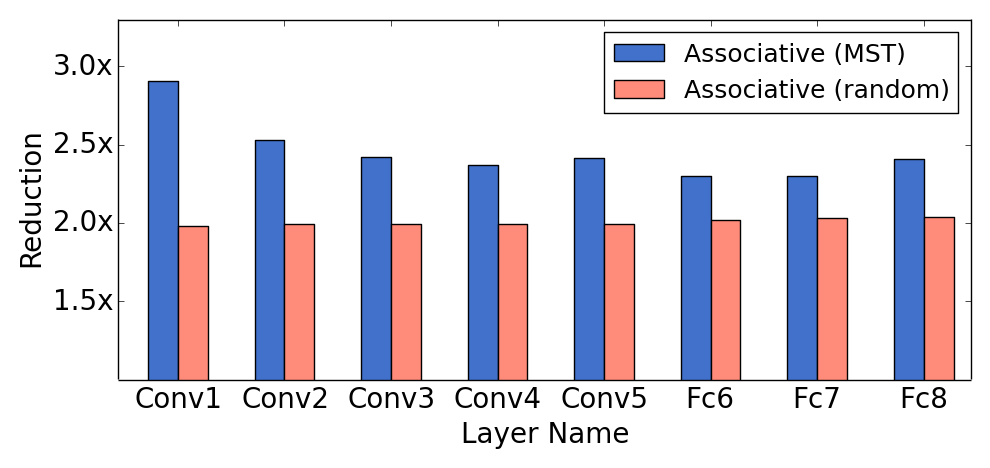}
\caption{The associative implementation of binary tensor convolutions helps to gain 2$\times$ to 3$\times$ reductions in the required number of FADDs on all learnable layers of "AlexNet\_sketch".}
\label{fig:6}
\vskip -0.1in
\end{figure}

\subsection{Whole-net Performance}\label{sec:exp3}

\begin{figure*}[t]
\begin{center}
\captionsetup[subfigure]{labelformat=empty}
\subfloat[]{
\includegraphics[width=0.225\textwidth]{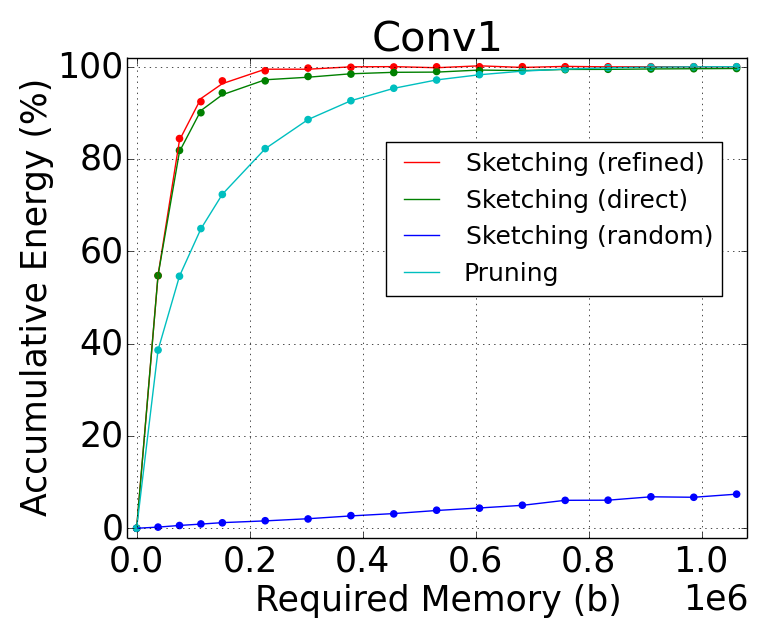}}\hskip 10pt
\subfloat[]{
\includegraphics[width=0.225\textwidth]{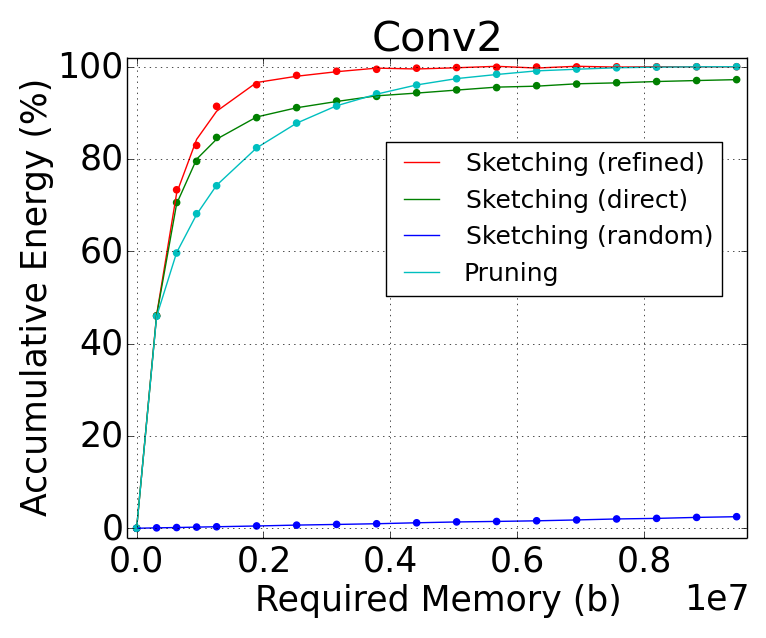}}\hskip 10pt
\subfloat[]{
\includegraphics[width=0.225\textwidth]{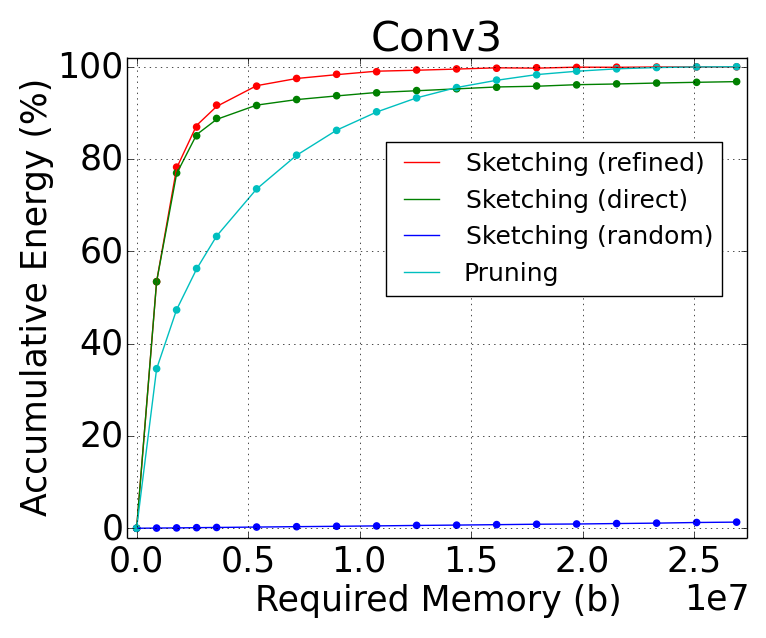}}\hskip 10pt
\subfloat[]{
\includegraphics[width=0.225\textwidth]{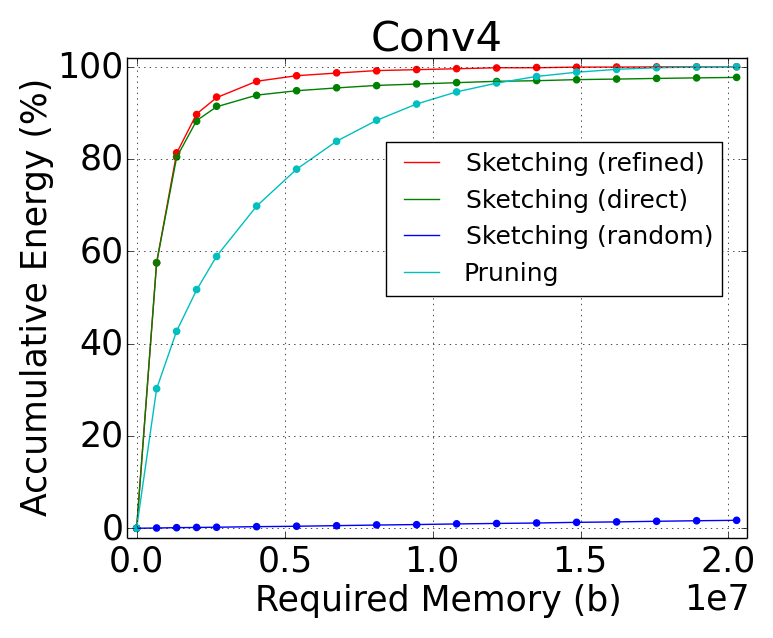}}\hskip 10pt
\vskip -10pt
\subfloat[]{
\includegraphics[width=0.225\textwidth]{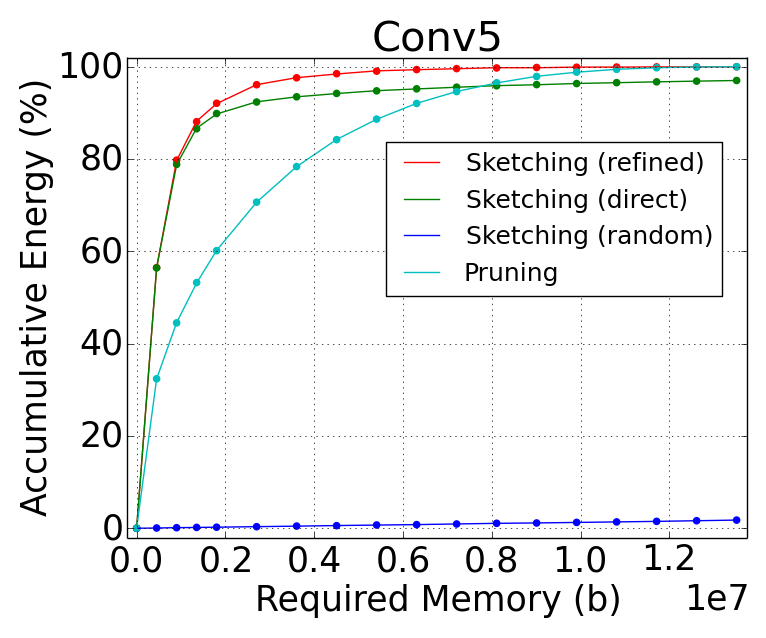}}\hskip 10pt
\subfloat[]{
\includegraphics[width=0.225\textwidth]{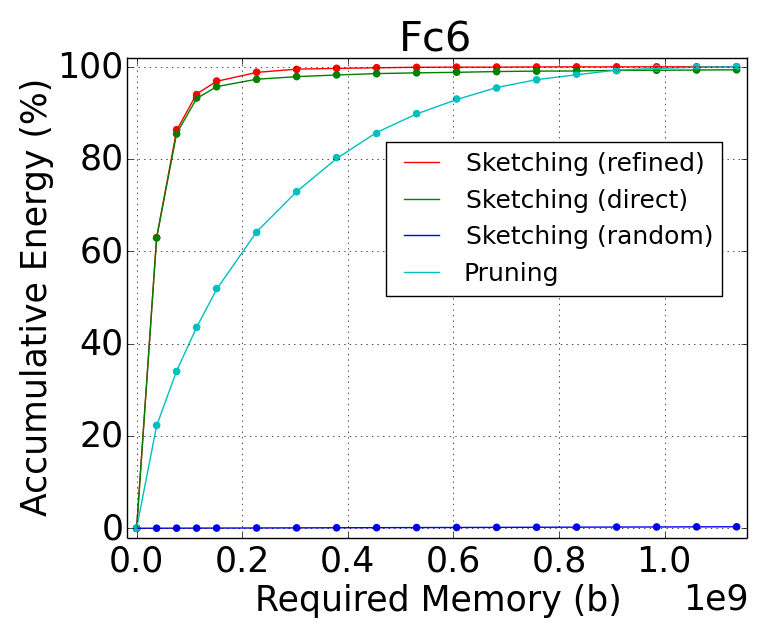}}\hskip  10pt
\subfloat[]{
\includegraphics[width=0.225\textwidth]{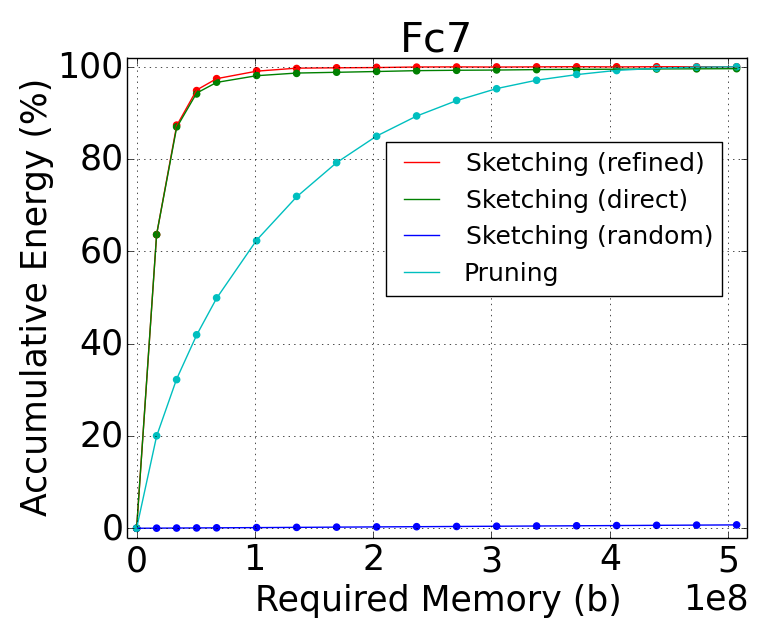}}\hskip 10pt
\subfloat[]{
\includegraphics[width=0.225\textwidth]{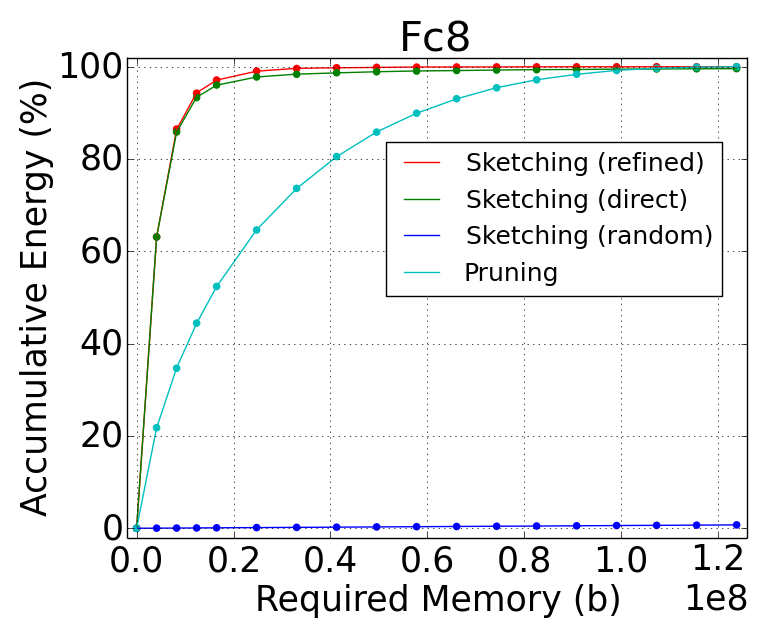}}\hskip 10pt
\caption{Network sketching approximates AlexNet well enough with a much smaller amount of committed memory, and the refinement operation helps to achieve better convergency on all of its learnable layers.}
\label{fig:5}
\end{center}
\vskip -0.06in
\end{figure*}

Having tested Algorithm~\ref{alg:1}, ~\ref{alg:2} and~\ref{alg:3} on the base of their own criteria, it is time to compare the whole-net performance of our sketch with that of other binary weight models~\cite{Courbariaux2015, Rastegari2016}.
Inspired by the previous experimental results, we still use the 3-bit (direct and refined) sketches for evaluation, as they are both very efficient and accurate.
Considering that the fully-connected layers of AlexNet contain more than 95\% of its parameters, we should try sketching them to an extreme, namely 1 bit.
Similar with Rastegari et al.~\cite{Rastegari2016}, we also keep the 'fc8' layer (or say, the output layer) to be of its full precision and report the top-1 and top-5 inference accuracies.
However, distinguished from their methods, we sketch the 'conv1' layer as well because it is also compute-intensive (as shown in Table~\ref{tab:1}).

\begin{table}[ht]
\begin{center}
\begin{tabular}{|C{0.8in}|c|c|c|}
\hline
Model & Params (b) & Top-1 (\%) & Top-5 (\%)   \\
\hline \hline
Reference &  $\smallsim$1951M & 57.2 & 80.3  \\
Sketch (ref.) & $\smallsim$193M & \textbf{55.2} & \textbf{78.8}  \\
Sketch (dir.) & $\smallsim$193M & 54.4 & 78.1  \\
BWN~\cite{Rastegari2016} & $\smallsim$190M & 53.8 & 77.0  \\
BC~\cite{Courbariaux2015}& $\smallsim$189M & 35.4 & 61.0 \\
\hline
\end{tabular}
\end{center}
\caption{Network sketching technique generates binary-weight AlexNets with the ability to make faithful inference and roughly \textbf{10.1}$\times$ fewer parameters than its reference (in bits). Test accuracies of the competitors are cited from the paper. An updated version of BWN gains significantly improved accuracies (top-1: 56.8\% and top-5: 79.4\%), but the technical improvement seems unclear.}\label{tab:5}
\end{table}

Just to avoid the propagation of reconstruction errors, we need to somehow fine-tune the generated sketches.
Naturally, there are two protocols to help accomplish this task; one is known as projection gradient descent and the other is stochastic gradient descent with full precision weight update~\cite{Courbariaux2015}.
We choose the latter by virtue of its better convergency.
The training batch size is set as 256 and the momentum is 0.9.
We let the learning rate drops every 20,000 iterations from 0.001, which is one tenth of the original value as set in Krizhevsky et al.'s paper~\cite{Krizhevsky2012}, and use only the center crop for accuracy evaluation.
After totally 70,000 iterations (i.e., roughly 14 epochs), our sketches can make faithful inference on the test set, and the refined model is better than the direct one.
As shown in Table~\ref{tab:5}, our refined sketch of AlexNet achieves a top-1 accuracy of 55.2\% and a top-5 accuracy of 78.8\%, which means it outperforms the recent released models in the name of BinaryConnect (BC)~\cite{Courbariaux2015} and Binary-Weight Network (BWN)~\cite{Rastegari2016} by large margins, while the required number of parameters only exceeds a little bit.

Network pruning technique also achieves compelling results on compressing AlexNet.
However, it demands a lot of extra space for storing parameter indices, and more importantly, even the optimal pruning methods perform mediocrely on convolutional layers~\cite{Han2015, Guo2016}.
In contrast, network sketching works sufficiently well on both of the layer types.
Here we also testify its efficacy on ResNet~\cite{He2016}.
Being equipped with much more convolutional layers than that of AlexNet, ResNet wins the ILSVRC-2015 classification competition.
There are many instantiations of its architecture, and for fair comparisons, we choose the type B version of 18 layers (as with Rastegari et al.~\cite{Rastegari2016}).

A pre-trained Torch model is available online~\footnote{\url{https://github.com/facebook/fb.resnet.torch/tree/master/pretrained}.} and we convert it into an equivalent Caffe model before sketching~\footnote{\url{https://github.com/facebook/fb-caffe-exts}.}.
For the fine-tuning process, we set the training batch size as 64 and let the learning rate drops from 0.0001.
After 200,000 iterations (i.e., roughly 10 epochs), the generated sketch represents a top-1 accuracy of 67.8\% and a top-5 accuracy of 88.4\% on ImageNet dataset.
Refer to Table~\ref{tab:6} for a comparison of the classification accuracy of different binary-weight models.

\begin{table}[ht]
\begin{center}
\begin{tabular}{|C{0.8in}|c|c|c|}
\hline
Model & Params (b) & Top-1 (\%) & Top-5 (\%)   \\
\hline \hline
Reference & $\smallsim$374M & 68.8 & 89.0  \\
Sketch (ref.)	& $\smallsim$51M  & \textbf{67.8} & \textbf{88.4}  \\
Sketch (dir.)	& $\smallsim$51M & 67.3 & 88.2  \\
BWN~\cite{Rastegari2016}& $\smallsim$28M  & 60.8 & 83.0  \\
\hline
\end{tabular}
\end{center}
\caption{Network sketching technique generates binary-weight ResNets with the ability to make faithful inference and roughly \textbf{7.4}$\times$ fewer parameters than its reference (in bits). The test accuracies of BWN are directly cited from its paper.}\label{tab:6}
\end{table}

\section{Conclusions}\label{sec:con}
In this paper, we introduce network sketching as a novel technology for pursuing binary-weight CNNs.
It is more flexible than the current available methods and it enables researchers and engineers to regulate the precision of generated sketches and get better trade-off between the model efficiency and accuracy.
Both theoretical and empirical analyses have been given to validate its efficacy.
Moreover, we also propose an associative implementation of binary tensor convolutions to further speedup the sketches.
After all these efforts, we are able to generate binary-weight AlexNets and ResNets with the ability to make both efficient and faithful inference on the ImageNet classification task.
Future works shall include exploring the sketching results of other CNNs.

{\small
\bibliographystyle{ieee}
\bibliography{egbib}
}

\end{document}